\documentclass[letterpaper, 10 pt, conference]{ieeeconf}  

\RequirePackage{amsmath}
\usepackage{amsfonts}

\RequirePackage{color}

\RequirePackage{graphicx}

\RequirePackage{algorithm}
\RequirePackage{algorithmic}

\RequirePackage{listings}

\RequirePackage{tikz}

\newcommand{\R}{\mathbb{R}}

\newcommand{\x}{\boldsymbol{x}}
\newcommand{\y}{\boldsymbol{y}}

\newcommand{\norm}[1]{\left\lVert#1\right\rVert}

\newtheorem{theorem}{Theorem}
\newtheorem{proposition}[theorem]{Proposition}

\newtheorem{remark}[theorem]{Remark}

\newtheorem{assumption}[theorem]{Assumption}
\newtheorem{lemma}[theorem]{Lemma}

\usetikzlibrary{shapes.geometric, arrows.meta}

\IEEEoverridecommandlockouts                              

\title{\LARGE \bf
A Robust Controller based on Gaussian Processes for Robotic Manipulators with Unknown Uncertainty
}

\author{Giulio Giacomuzzo$^{1}$, Mohamed Abdelwahab$^{1}$, Marco Calì$^{1}$, Alberto Dalla Libera$^{1}$, Ruggero Carli$^{1}$
\thanks{$^{1}$Department of Information Engineering, Università di Padova, Italy
}
\thanks{$^{2}$Alberto Dalla Libera and Giulio Giacomuzzo were supported by PNRR research activities of the consortium iNEST (Interconnected North-East Innovation Ecosystem) funded by the European Union Next GenerationEU (Piano Nazionale di Ripresa e Resilienza (PNRR) – Missione 4 Componente 2, Investimento 1.5 – D.D. 1058  23/06/2022, ECS\_00000043). This manuscript reflects only the Authors’ views and opinions, neither the European Union nor the European Commission can be considered responsible for them.\newline
{Corresponding author email: \tt\small giacomuzzo@dei.unipd.it}}%
}%

\begin{document}

\maketitle
\thispagestyle{empty}
\pagestyle{empty}

\begin{abstract}
In this paper, we propose a novel learning-based robust feedback linearization strategy to ensure precise trajectory tracking for an important family of Lagrangian systems. We assume a nominal knowledge of the dynamics is given but no a-priori bounds on the model mismatch are available. In our approach, the key ingredient is the adoption of a regression framework based on Gaussian Processes (GPR) to estimate the model mismatch. This estimate is added to the outer loop of a classical feedback linearization scheme based on the nominal knowledge available. Then, to compensate for the residual uncertainty, we robustify the controller including an additional term whose size is designed based on the variance provided by the GPR framework. We proved that, with high probability, the proposed scheme is able to guarantee asymptotic tracking of a desired trajectory. We tested numerically our strategy on a 2 degrees of freedom planar robot.  

\end{abstract}


\section{Introduction}

Precise and accurate trajectory tracking is of paramount importance in robotics, with applications in different fields, ranging from manufacturing to field and medical robotics. Nonetheless, guaranteeing high tracking accuracy is a challenging control problem, particularly when complex systems with non-linear dynamics are considered, as in the case of robotic manipulators.

A widespread approach to tackle the problem is feedback linearization (FL) control \cite{siciliano2010robotics}. FL controllers combine two nested loops: an outer loop that compensates for non-linearities, reducing the dynamics to be controlled to a linear system, and an inner loop typically based on a PD action meant to stabilize the error dynamics. 
Despite being effective, this approach requires precise knowledge of the system dynamics, which is unlikely to be obtained in practice.

To address this limitation, robust FL solutions have been explored, which exploit a-priori bounds on the model mismatch as described in section 8.5.3 of \cite{siciliano2010robotics}.
When these bounds are unknown, however, the design of robust FL controllers is still an open issue, worth of further research.

While few solutions have been proposed to tune online the size of the uncertainty \cite{abdelwahab2024adaptive}, a large part of the literature focused on compensating unknown or inaccurate models by learning the system dynamics from data, in a black-box fashion \cite{lutter2019deep, ADL_GIP19, giacomuzzo2024black}. An interesting class of methods is based on Gaussian Process Regression (GPR) \cite{rasmussen2003gaussian, care2023kernel}. Differently from other widespread regression frameworks such as Neural Networks, GPR provides a bound on the uncertainty of the estimates. This additional information can be exploited to robustify FL controllers. 

Several works proposed the use of GPR to learn the robot's dynamics in trajectory-tracking applications. In \cite{beckers2019stable} the authors develop a computed torque control scheme, where they estimate the modeling error with GPR and use the posterior variance to vary the feedback gains. Despite being effective, however, the computed torque control law is not guaranteed to stabilize the error dynamics.
In \cite{dalla2021control}, the authors learn the inverse dynamics of manipulators through GPR and exploit the learned model to develop a standard FL scheme, with no robust action. Finally, in \cite{helwa2019provably} the authors propose to learn the error in the commanded acceleration through GPR, and then use the posterior variance as a bound on the model uncertainty in the robust FL framework. However, they use GPR only to approximate the model uncertainty, without using the learned dynamics to improve the model. 

In this paper, we propose a novel robust FL controller based on GPR, where we exploit both the learned model and the posterior variance to compensate for modeling inaccuracies and guarantee tracking error rejection. We assume a nominal knowledge of the system's dynamics to be available, and we use GPR to learn the modeling error between the nominal dynamics and the measured one. We add this estimate to the outer loop of a classic FL scheme to improve the compensation of non-linearities, and then we design a robust term based on the GPR variance to compensate for the residual uncertainty.

We provide a theoretical analysis of convergence of the proposed approach, based on the Lyapunov direct method. Together, we verify the effectiveness of the method on a simulated  2 degrees of freedom (DOF) robotic arm. Results confirm the capability of our scheme to compensate for even severe modeling inaccuracies, guaranteeing accurate tracking performance.

The paper unfolds as follows. In Section \ref{sec:Prob_For} we state the problem of interest. In Section \ref{sec:back_GPR} we briefly review the GPR framework, while in Section \ref{sec:method} we describe the novel robust FL scheme we develop, together with the theoretical analysis of its convergence properties. In Section \ref{sec:NumericalResults} we illustrate the numerical results we obtained. Finally Section \ref{sec:conclusions} concludes the paper.

\section{Problem Statement}\label{sec:Prob_For}

Consider the class of Lagrangian systems described by the following dynamics
\begin{equation}\label{eq:Dynamics}
M(q(t)) \ddot{q}(t) + C(q(t), \dot{q}(t)) \dot{q}(t) + g(q(t)) = \tau(t),
\end{equation}
where $q=[q_1,\ldots, q_N]^T$ is the $N$-th dimensional vector of generalized coordinates ($q_i$ can be either a displacement or an angle), and $\dot{q}=[\dot{q}_1,\ldots, \dot{q}_N]^T$ and $\ddot{q}=[\ddot{q}_1,\ldots, \ddot{q}_N]^T$ are the generalized velocities and accelerations, respectively; $M(q(t)) \in \mathbb{R}^{N \times N}$ is the positive definite inertia matrix; $C(q(t), \dot{q}(t)) \in \mathbb{R}^{N \times N}$ incorporates the effects of the Coriolis and centrifugal forces; $g(q) \in \R^N$ are the forces due to gravity, while $\tau=[\tau_1,\ldots, \tau_N]^T$ represents the vector of the generalized torques applied to the system.

The goal for system in \eqref{eq:Dynamics} is that of tracking a desired trajectory $\left(q^d(t), \dot{q}^d(t), \ddot{q}^d(t)\right)$, that is, to drive the position error $\tilde{q}(t)=q^d(t)-q(t)$ to $0$. In the remainder of the paper, when there is no risk of confusion, we omit the dependence on time $t$ in the aforementioned vectors.  

Let us start our analysis by assuming that a perfect knowledge of the quantities $M(q), C(q,\dot{q}),g(q)$ is available. In this context, an efficient solution proposed in the literature is the FL control law based on two nested feedback loops. The outer loop consists on the following feedback linearization step 
\begin{equation}\label{eq:Feedback_Lin}
\tau = M(q) a + n(q, \dot{q}), 
\end{equation}
where $n(q, \dot{q})$ include the Coriolis/centrifugal and gravitational effects, i.e., $=C(q,\dot{q}) \dot{q} + g(q)$ and
where $a$ is an auxiliary control. The control in \eqref{eq:Feedback_Lin} allows to eliminate the nonlinearities of the system. Indeed, by substituting \eqref{eq:Feedback_Lin} into \eqref{eq:Dynamics} we obtain the double integrator dynamics 
$$
\ddot{q}=a,
$$
which can be asymptotically stabilized by designing, within the inner loop, the auxiliary input $a$ as the sum of a feedforward term and a PD controller, that is,
\begin{equation}\label{eq:auxiliary_law}
a = \ddot{q}_d + K_P \tilde{q} + K_D \dot{\tilde{q}}, 
\end{equation}
where $\tilde{q}=q_d - q$ is the position error, $\dot{\tilde{q}}=\dot{q}_d -\dot{q}$ is the velocity error, and $K_P, K_D$ are, respectively, the proportional and derivative gain matrices. 

By replacing \eqref{eq:auxiliary_law} into \eqref{eq:Feedback_Lin} we obtain the overall control scheme
\begin{equation}\label{eq:Fed_Lin_Scheme}
\tau = M(q)(\ddot{q}_d + K_P \tilde{q} + K_D \dot{\tilde{q}}) + n(q, \dot{q}).
\end{equation}
It is easy to see that, by applying \eqref{eq:Fed_Lin_Scheme}, the dynamics of the position error are ruled by the following second-order differential equation 
$$
\ddot{\tilde{q}} + K_D \dot{\tilde{q}} +K_P \tilde{q}=0.
$$
Assuming both $K_P$ and $K_D$ to be positive definite matrices, it results that $\tilde{q}$ converges exponentially to zero independently from the initial condition $\tilde{q}(0), \dot{\tilde{q}}(0)$. Note that the crucial assumption to guarantee the validity of the above control scheme is that a perfect model knowledge is available. In practice, this assumption is never satisfied and a model uncertainty is always present. 

From now on, let us assume that only estimates $\hat{M}$, $\hat{C}$, $\hat{g}$ of $M,C,g$ are given, where in general $\hat{M} \neq M$, $\hat{C}\neq C$, $\hat{g}\neq g$, and $\hat{M}$ is positive definite. Based on the available knowledge, the control law \eqref{eq:Fed_Lin_Scheme} is modified as
\begin{equation}\label{eq:Fed_Lin_Scheme_uncertaint}
\tau = \hat{M}(q)(\ddot{q}_d + K_P \tilde{q} + K_D \dot{\tilde{q}}) + \hat{n}(q, \dot{q}),
\end{equation}
where $\hat{n} = \hat{C} \dot{q} + \hat{g}$. As a consequence, the error dynamics are now ruled by
$$
\ddot{\tilde{q}} + K_D \dot{\tilde{q}} +K_P \tilde{q}=\eta,
$$
where
\begin{equation}\label{eq:eta}
\eta=(I-M^{-1}\hat{M}) a - M^{-1} \tilde{n},
\end{equation}
with $\tilde{n}=\hat{n}-n$ and $a$ defined as in \eqref{eq:auxiliary_law}. The variable $\eta$ accounts for the model uncertainty. In this case only a practical stability objective can be attained.

In \cite{siciliano2010robotics}, in order to provide robustness against the model uncertainties, the auxiliary law is modified as
\begin{equation}\label{eq:modified_auxiliary_law}
a = \ddot{q}_d + K_P \tilde{q} + K_D \dot{\tilde{q}}+ w, 
\end{equation}
leading to 
\begin{equation}\label{eq:Fed_Lin_Scheme_robust}
\tau = \hat{M}(q)\left(\ddot{q}_d + K_P \tilde{q} + K_D \dot{\tilde{q}} + \rho \frac{z}{\|z\|}\right) + \hat{n}(q, \dot{q}),
\end{equation}
where the additional term $w$ is designed using the Lyapunov direct method as we describe next. Let 
$$
\xi =\begin{bmatrix}
        \tilde{q}  \\
        \dot{\tilde{q}}
    \end{bmatrix} \,\,\in\,\, \mathbb{R}^{2N},
$$
and let
$$
\tilde{H}=\left[
\begin{array}{cc}
0 & I \\
-K_P & -K_D
\end{array}
\right
], \qquad 
D=\left[
\begin{array}{c}
0 \\
I
\end{array}
\right
].
$$
By plugging \eqref{eq:modified_auxiliary_law}
into \eqref{eq:eta}, one can obtain the following description for the dynamics of $\xi$ 
\begin{equation}
\dot{\xi}= \tilde{H} \xi +D (\eta- w).
\end{equation}
Now let 
$$
V(\xi)=\xi^T Q \xi,
$$
be a candidate Lyapunov function 
where $Q$ is a $(2N \times 2N)$ positive definite matrix 
In \cite{siciliano2010robotics}, it is shown that by designing
$$
w=\frac{\rho}{\|z\|}z, 
$$
with $z=D^TQ\xi$ and $\rho$ such that $\rho \geq \|\eta\|$ then $\dot{V}<0$ for all $\xi \neq 0$, that is, the time derivative of $V$, i.e., 
is less than zero for all $\xi \neq 0$, thus ensuring that $\xi \to 0$ asymptotically.

Typically, in literature the following assumptions are made

\begin{equation}\label{eq:ass1}
\left\| I-{M}^{-1}\left( q\right) \hat M\left( q\right) \right\| \leq \alpha \leq 1,  \quad \forall\, q,
\end{equation}
 
\begin{equation}\label{eq:ass2}
    \left\| \widetilde{n}\left( q,\dot q\right) \right\| \leq \Phi <  \infty \quad \forall \,q, \dot q,
\end{equation}

\begin{equation}\label{eq:ass3}
  \sup_{t \geq 0}  \left\| \ddot q_{d} \right\| < Q_{M} <  \infty  \quad \forall\,  \ddot q_d,
\end{equation}

\begin{equation}\label{eq:ass4}
0< M_{min} \leq \|M^{-1}(q)\| \leq M_{max} < \infty, \qquad \forall\, q,
\end{equation}
where the parameters $\alpha, \Phi, Q_M, M_{max}$ are assumed to be a-priori known. Interestingly,  by designing $\rho$ in such a way that
\begin{equation}\label{eq:rho}
\rho> \frac{1}{1-\alpha} \left(\alpha Q_M+\alpha \|K\|\,\|\xi\| +M_{max} \Phi\right),
\end{equation}
where $K=\left[K_P \,\,\,K_D\right]$,
it holds that $\rho > \|\eta\|$ for all $q, \dot{q}, \ddot{q}_d$, and, in turn, $\dot{V} <0$.



Observe that the right hand side of \eqref{eq:rho} relies on the a-priori information of the bounds on the size of the model uncertainties given in \eqref{eq:ass1}, \eqref{eq:ass2}, \eqref{eq:ass3} and \eqref{eq:ass4}. If this knowledge is not available, the design of $\rho$ becomes more challenging.

In this paper we propose a novel approach which includes two main ingredients. First, we adopt the GPR framework to estimate the model mismatch. This estimate is added to the outer loop of the classical feedback linearization scheme in \eqref{eq:Fed_Lin_Scheme_uncertaint} to reduce the uncertainty to be compensated. Second, an additional term is included in the outer loop to counteract the residual uncertainty. The design of this term is driven by a Lyapunov argument (to guarantee that $\dot{V}$ is negative) and, interestingly, exploits the information given by the variance of the estimate provided by the GPs.

\section{Gaussian Process Regression} \label{sec:back_GPR}
Gaussian Process Regression \cite{rasmussen2003gaussian} is a probabilistic framework that allows to estimate an unknown function directly from input-output data. Assume the unknown function to be $f: \R^{p} \rightarrow \R$, and let a set of n input-output observations $\mathcal{D} = \{X,\y\}$ to be available. Here, $X = \{\x_1, \dots, \x_n\}$ contains the $n$ input locations, and $\y \in \R^n$ collects the corresponding $n$ noisy measurements. In particular, the measurements $\y$ are assumed to be generated by the following probabilistic model
\begin{equation}
    \y = \begin{bmatrix}f(\x_1) \\ \vdots \\ f(\x_n)\end{bmatrix} + \begin{bmatrix}w_1 \\ \vdots \\ w_n\end{bmatrix} = \boldsymbol{f}(X) + \boldsymbol{w}, 
\end{equation}
where $w_i$ are instances of i.i.d Gaussian noise with standard deviation $\sigma$, while $f$ is modeled a priori as a zero-mean Gaussian Process, namely $\boldsymbol{f}(X) \sim \mathcal{N}(0, \mathbf{K}$). The covariance matrix $\mathbf{K}$, is named \textit{kernel matrix}, and its elements are defined through a \textit{kernel function} $k(\cdot, \cdot):\mathbb{R}^m \times \mathbb{R}^m \rightarrow \mathbb{R}$.
In particular, the element of $\mathbf{K}$ in position $(h,j)$ is obtained as $k(\x_h, \x_j)$. 

Given any new input location $\x_* \in \R^p$, it can be proved that the posterior distribution of $f(\x_*)$ given $\mathcal{D}$ is itself Gaussian, with mean 
\begin{equation}
    \label{eq:posterior_mean}
    \mu\left(f | \x_*, \mathcal{D}\ \right) = \boldsymbol{k}^{T}_* \boldsymbol{\alpha}
\end{equation}
and covariance
\begin{equation}
    \label{eq:posterior_var}
    \Sigma\left( f | \x_*, \mathcal{D}\right)= k(\x_*, \x_*) - \boldsymbol{k}^{T}_*(\mathbf{K} + \sigma^2\,I)^{-1}\boldsymbol{k}_*,
\end{equation}
where 
\begin{equation*}
    \boldsymbol{k}_* = [k(\x_*, \x_1), \dots, k(\x_*, \x_n)]^T
\end{equation*}
and
\begin{equation*}
   \boldsymbol{\alpha} = (\mathbf{K} + \sigma^2 I)^{-1}\y.
\end{equation*}
Since the posterior distribution of $f(x_*)$  is Gaussian, its maximum a posteriori estimator coincides with its mean $\mu\left(f | \x_*, \mathcal{D}\ \right)$, while the variance $ \Sigma\left( f | \x_*, \mathcal{D}\right)$ represents a measure of the model uncertainty when predicting $f(x_*)$. 

As it can be seen from \eqref{eq:posterior_mean} and \eqref{eq:posterior_var}, the choice of the kernel function represents the crucial aspect within the GPR framework. A common choice for $k(\cdot, \cdot)$ is the Square Exponential (SE) kernel, which defines the covariance between samples according to the distance between their input locations. More formally
\begin{equation} \label{eq:K-SE}
    k_{SE}(\x, \x') = \lambda e^{-\norm{\x - \x'}^2_{\Sigma^{-1}}},
\end{equation}
where $\lambda$ and $\Sigma$ are the kernel hyperparameters, which are typically learned from data through marginal likelihood maximization, see \cite{rasmussen2003gaussian}.

So far we considered the case where the unknown function is scalar, namely $f: \R^p \rightarrow \R$. 
When GPR is applied to vector-valued functions, namely when $f: \R^p \rightarrow \R^d$, the standard approach consists in assuming the $d$ components of $f$ to be conditionally independent given the GP input $\x$. 
As a consequence, the overall $d$-dimensional regression problem is split into a set of $d$ scalar and independent GPR problems,
$$
\y^i = f^i(X) + \boldsymbol{w}^i
$$
where the $i$-th component $f^i : \mathbb{R}^m \rightarrow \mathbb{R}$ is estimated independently of the others as described in the preceding part of this Section,  with $ \boldsymbol{y}^i = [y_1^i \dots y_n^i]^T$ , being $y_t^i$ a measure of the $i$-th component of $f$ at input location $\x_t$.

\section{GPR-based robust feedback linearization control}\label{sec:method}

In this Section, we describe the novel method we propose. 
Let $e(q,\dot{q},\ddot{q})$ denote the mismatch model, that is, 
$$
e(q,\dot{q},\ddot{q})= \tilde{M}(q)\ddot{q} + \tilde{C}(q,\dot{q})\dot{q} +\tilde{g}(q),
$$
where $\tilde{M}=M-\hat{M}$, $\tilde{C}=C-\hat{C}$ and $\tilde{g}=g-\hat{g}$.
Notice that 
\begin{equation}\label{eq:e}
    e(q,\dot{q},\ddot{q})=\tau-\left(\hat{M}(q)\ddot{q}+\hat{C}(q, \dot{q})\dot{q}+\hat{g}(q)\right)
\end{equation}
In our approach, the quantity $e$ is estimated by Gaussian regression as next described. Suppose we have collected the following $n$ samples $\left\{{\bf q^{(i)}}, \tau_{\text{m}}^{(i)}\right\}_{i=1}^n$ where ${\bf q^{(i)}}=\left(q^{(i)}, \dot{q}^{(i)}, \ddot{q}^{(i)}\right)$ denotes the robot configuration of the $i$-th sample and $\tau_{\text{m}}^{(i)}$ is the measurement of the corresponding torque $\tau^{(i)}$, i.e., $\tau_{\text{m}}^{(i)}=\tau^{(i)}+\omega^{(i)}$ being $\omega^{(i)}$ a zero-mean Gaussian noise of variance $\sigma^2$.

Let $e^{(i)}$ be the evaluation of $e$ in $q^{(i)}$, i.e.,
    $$e^{(i)}=\tau^{(i)}-\left(\hat{M}\left(q^{(i)}\right)\ddot{q}^{(i)}+\hat{C}\left(q^{(i)}, \dot{q}^{(i)}\right)\dot{q}^{(i)}+\hat{g}\left(q^{(i)}\right)\right)$$
and, accordingly, let $e^{i}_{\text{m}}$ its noisy version, i.e., $e^{i}_{\text{m}}=e^{(i)}+\omega^{(i)}$.

The regression procedure illustrated in Section \ref{sec:back_GPR}, can be applied in the following way to predict $e$ in any configuration ${\bf q}=\left(q, \dot{q}, \ddot{q}\right)$, based on data
$$
\mathcal{D}=\left\{{\bf q^{(i)}}, e_{\text{m}}^{(i)}\right\}_{i=1}^n.
$$
Let $e_1,\ldots, e_N$ be the $N$ components of $e$ that we assume to be conditionally independent given the GP input ${\bf q}$. Then, using formulas in \eqref{eq:posterior_mean} and \eqref{eq:posterior_var}, we can compute
$$
\mu\left(e_i|{\bf q}, \mathcal{D} \right), \qquad \Sigma\left(e_i|{\bf q}, \mathcal{D}\right),
$$
for $i=1,\ldots,N$.
Now let $\hat{e}({\bf q})$ be the prediction of $e$ in ${\bf q}$, then
$$
\hat{e}({\bf q})=\left[
\begin{array}{c}
\mu\left(e_1|{\bf q}, \mathcal{D} \right)\\
\vdots \\
\mu\left(e_N|{\bf q}, \mathcal{D} \right)
\end{array}
\right]
$$
with the corresponding variance given by 
$$
\left[
\begin{array}{ccc}
\Sigma\left(e_1|{\bf q}, \mathcal{D}\right) & &\\
&\ddots & \\
& & \Sigma\left(e_N|{\bf q}, \mathcal{D} \right)
\end{array}
\right].
$$

The control law we propose is the following
\begin{equation}\label{eq:Fed_Lin_Scheme_robust_prob}
\tau = \hat{M}(q)\left(\ddot{q}_d + K_P \tilde{q} + K_D \dot{\tilde{q}}\right) + \hat{n}(q, \dot{q})+\hat{e}(q, \dot{q}, \ddot{q})+w,
\end{equation}
where $w=\rho \frac{z}{\|z\|}$ with 
$$
z= \hat{M}^{-1}(q)D^TQ\xi.
$$
Next, we provide a first result that characterizes the convergence property of \eqref{eq:Fed_Lin_Scheme_robust_prob}.

\begin{proposition}
Consider the Lagrangian system in \eqref{eq:Dynamics}. Assume that 
for any configuration ${\bf q}(t)=(q(t), \dot{q}(t), \ddot{q}(t))$, it holds
\begin{equation}\label{eq:Bound_rho}
    \rho(t) > \| e({\bf q}(t)) - \hat{e}({\bf q}(t))\|.
\end{equation}
Then, with the control law in \eqref{eq:Fed_Lin_Scheme_robust_prob} the tracking error $\tilde{q}$ goes asymptotically to zero. 
\end{proposition}
\begin{proof}
By plugging \eqref{eq:Fed_Lin_Scheme_robust_prob} into \eqref{eq:Dynamics}, after some standard algebraic manipulations, we obtain
$$
\hat{M}(q)\left(\ddot{\tilde{q}}(t) + K_D \dot{\tilde{q}}+K_P \tilde{q}\right) + w + \tilde{e} =0
$$
where $\tilde{e}=e -\hat{e}$. Then we have that
$$
\ddot{\tilde{q}} +K_D \dot{\tilde{q}}+K_P\tilde{q}= \hat{M}^{-1}(\tilde{e}-w).
$$
Introducing $\xi=\left[\tilde{q}^T\,\,\,\,\dot{\tilde{q}}^T\right]^T$ we can rewrite the above differential equation as
$$
\dot{\xi} = \tilde{H} \xi + D \hat{M}^{-1}(\tilde{e}-w).
$$
Let us consider the Lyapunov function $V=\xi^T Q \xi$ and let us compute its time derivative which after some calculations results to be
\begin{align*}
 \dot{V}&= \xi^T(QH+HQ) \xi + 2 \xi^TQD \hat{M}^{-1}(q) \left(\tilde{e}-w\right)   \\ 
 &=-\xi^TP \xi + 2 \xi^TQD \hat{M}^{-1}(q) \left(\tilde{e}-w\right)
\end{align*}
where $P=-(HQ+HQ)$ is a definitive positive matrix. Now since $z= \hat{M}^{-1}(q)D^TQ\xi$ then
$$
 \dot{V}= -\xi^T P\xi + 2 z^T\left(\tilde{e}-w\right).
$$
Recall that $w=\rho \frac{z}{\|z\|}$, hence
$$
z^T\tilde{e}-z^T w=z^T\tilde{e}-\rho \|z\|. 
$$
and, from this last expression it is easy to see that, if $\rho > \|\tilde{e}\|$ then  $\dot{V}<0$ if $\xi \neq 0$. This concludes the proof.
\end{proof}
Based on the above result, the key point is now to provide a design rule for $\rho(t)$ that guarantees to satisfy the condition in \eqref{eq:Bound_rho}. To this aim, in the next Section we will exploit the knowledge of the variances $\Sigma\left(e_i|{\bf q}, \mathcal{D}\right)$, $i=1,\ldots,N$, to obtain a probabilistic model for $| e_i({\bf q}(t)) - \hat{e}_i({\bf q}(t))|$ and, in turn, an expression for $\rho(t)$.

\begin{remark}
It is worth stressing two substantial differences of \eqref{eq:Fed_Lin_Scheme_robust_prob} with respect to \eqref{eq:Fed_Lin_Scheme_robust}. Firstly, we add in the outer loop the term $\hat{e}$, that is, the estimate of the model mismatch $e$; intuitively this should reduce the uncertainties to be compensated by the robust component. Secondly, the term $w$ is designed in a different way; indeed $w$ is added to the outer loop (and not to the inner loop) and the vector $z$ is obtained by premultiplying $D^TQ\xi$ by the inverse of $\hat{M}$.
\end{remark}

\subsection{Design rule for $\rho$}

Before presenting the main convergence result for the control law we propose in this paper, we introduce an assumption and a preliminary result which are standard in the literature, see for example \cite{helwa2019provably}.   

\begin{assumption}\label{ass:1}
    The function $e_i({\bf q})$, $j=1,\ldots,N$, has a bounded norm $\|e_i\|_k$ with respect to the kernel $k(x,x')$ of the GP, and the noise $\omega_j$ added to the measurement is uniformly bounded by $\bar{\sigma}$.
\end{assumption}

\begin{lemma}\label{Lem:1}
    Suppose that Assumption \ref{ass:1} holds true. Let $\delta\in (0,1)$ and let $\mathcal{Q}$ be a compact set. Then, for $i=1,\ldots,N$, 
    \begin{align*}
         &   \text{Pr} \left\{|\mu(e_i|{\bf q}, \mathcal{D})-e_i({\bf q})| \leq \sqrt{\beta_i\,\Sigma\left(e_i|{\bf q}, \mathcal{D}\right)}\,,\,\, \text{for all ${\bf q} \in \mathcal{Q}$}  \right\} \\
         &\qquad \geq 1-\delta,
    \end{align*}
    where $\text{Pr}$ stands for probability and
    $$
    \beta_i = 2 \|e_i\|^2_k + 300 \gamma \ln^3((n+1)/\delta).
    $$
    The variable $\gamma \in \mathbb{R}$ is the maximum information gain and is given by
    $$
    \gamma = \max_{{{\bf q}^{(1)}, \ldots, {\bf q}^{(n+1)}}\in \mathcal{Q}} 0.5 \log(\text{det}(I+\bar{\sigma}^{-2} {\bf K}_{n+1})),
    $$
    where $\text{det}$ is the matrix determinant, $I \in \mathbb{R}^{(n+1)\times (n+1)}$ is the identity matrix, and
    ${\bf K}_{n+1} \in \mathbb{R}^{(n+1)\times (n+1)}$ is the covariance matrix given by 
    $[{\bf K}_{n+1}]_{(i,j)}=k(x_i,x_j)$, $i,j \in \left\{1,\ldots, n+1\right\}$. 
\end{lemma}

As stressed in \cite{helwa2019provably}, finding the information gain maximizer can be approximated by an efficient greedy algorithm, see for instance \cite{srinivas2012information}. 

Now, the key point here is that, thanks to the above Lemma, it is possible to define a confidence interval around the mean $\mu(e_i|{\bf q}, \mathcal{D})$ that is guaranteed to be correct for all the points ${\bf q} \in \mathcal{Q}$ with probability higher than $(1-\delta_p)$ where $\delta_p$ is typically supposed to be very small. More specifically, let ${\bf q}(t)= (q(t), \dot{q}(t), \ddot{q})(t)$ be the configuration of the system at time $t$ then we select the upper bound on 
$$
|\mu_i(e_i|{\bf{q}}(t), \mathcal{D})- e_i({\bf q}(t))|
$$
to be
\begin{align*}
\rho_i(t)= &\max \left\{|\mu_i(e_i|{\bf{q}}(t), \mathcal{D}) - \beta_i\,\Sigma\left(e_i|{\bf q}, \mathcal{D}\right)|, \right.\\
&\qquad\qquad \qquad \left. |\mu_i(e_i|{\bf{q}}(t), \mathcal{D})+\beta_i\,\Sigma\left(e_i|{\bf q}, \mathcal{D}\right)|\right\}.
\end{align*}
Then a good estimate for the upper bound on $\| e({\bf q}(t)) - \hat{e}({\bf q}(t))\|$ is given by
\begin{equation}\label{eq:rho}
    \rho(t)= \sqrt{\sum_{i=1}^N \rho_i^2(t)}.
\end{equation}

\subsection{Convergence analysis}
We start our analysis with the following assumption.
\begin{assumption}\label{ass:2}
   The desired trajectory to be tracked and the trajectory generated by the robot belong to a compact set.
\end{assumption}
\vspace{0.2cm}
Based on the above assumption, we are now ready to state the following main result. 
\begin{proposition}
    Consider the Lagrangian system in \eqref{eq:Dynamics} and assume Assumptions \ref{ass:1} and \ref{ass:2} hold true.  Then, the proposed,
robust, learning-based control strategy in \eqref{eq:Fed_Lin_Scheme_robust_prob}
with the uncertainty upper bound $\rho$ calculated by \eqref{eq:rho} ensures with high
probability of at least $(1 - \delta_p)^N$
that the tracking error $\tilde{q}(t)$ goes asymptotically to $0$.
\end{proposition}
\begin{proof}
We start by recalling that by Assumption \ref{ass:2} we have that ${\bf q}(t)$ lies in a compact set for any $t$. Now our problem reduces to computing the probability that $\rho(t)$ in \eqref{eq:rho} is a correct upper bound on $\| e({\bf q}(t)) - \hat{e}({\bf q}(t))\|$ for all $t$.

Consider the
confidence region proposed in Lemma \ref{Lem:1} for calculating
the upper bound on the absolute value of each component 
$|\mu(e_i|{\bf q}, \mathcal{D})-e_i({\bf q})|$, $i=1,\ldots,N$.
Then, under Assumption \ref{ass:2}, the probability that this upper
bound is correct for any $t$ is higher than $(1-\delta)$ from
Lemma \ref{Lem:1}. Since the $N$ Gaussian processes modeling $e_i$, $i=1,\ldots,N$, are assumed to be independent and the
noises added to the output observations are uncorrelated, then the
probability that the upper bounds on the absolute values of
all $|\mu(e_i|{\bf q}, \mathcal{D})-e_i({\bf q})|$, $i=1,\ldots,N$, and, in turn, the upper bound on $\| e({\bf q}(t)) - \hat{e}({\bf q}(t))\|$,
are correct is higher than $(1 - \delta)^N$.
This concludes the proof.
\end{proof}
\begin{remark}
   Though assumption \ref{ass:2} is realistic in practical scenarios, from a theoretical point of view it requires two strong facts to be true. In our future research we envision to relax Assumption \ref{ass:2}, by assuming for instance that only the desired trajectory lies in a compact and proving the proposed control strategy will not cause the actual trajectory to blow up. 
\end{remark}

\section{Simulation results}\label{sec:NumericalResults}
\begin{figure*}[ht!]
    \centering
    \includegraphics[width=.98\textwidth]{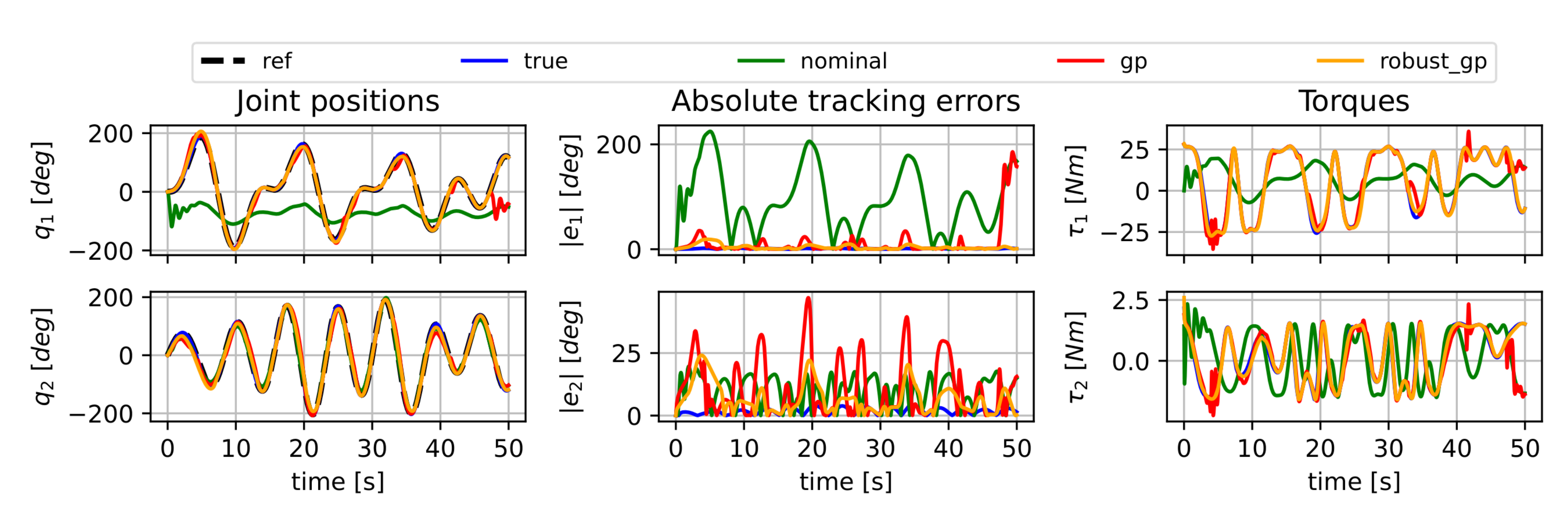}
    \caption{Trajectories obtained on one seed of the tracking experiments. The first column reports joint positions, the second column reports the absolute value of the tracking error, while the third column reports the actuation torques.}
    \label{fig:tracking_performance}
\end{figure*}

In this Section, we test the proposed method on a simulated setup involving a 2 DOF manipulator with 2 revolute joints. The robot's dynamics have been simulated in Python using the Scipy library \cite{2020SciPy-NMeth}. 


To verify our approach, we designed a trajectory-tracking experiment where the reference trajectory of each joint is the sum of random sinusoids, namely
\begin{equation}\label{eq:ref_trj}
    q^d_i(t) = \frac{2 \, \pi}{N_s} \sum_{i=1}^{N_s} \sin(\omega_i t),
\end{equation}
where $q^d_i$ denotes the desired position for the $i$-th joint, $N_s = 5$ is the number of sinusoids, while the angular frequency $\omega_i$ is randomly sampled from a uniform distribution, in such a way that $\omega_i \sim U([\omega_{min}, \omega_{max}])$, with $\omega_{min} = 0.1\pi $ rad/s and $\omega_{max}= 0.3\pi $ rad/s. 
All the reference trajectories last for $50$ seconds, and all the control loops run at a frequency of $100$ Hz.

To obtain statistically relevant results, we performed the trajectory tracking experiment on 10 different trajectories, obtained as in \eqref{eq:ref_trj} by varying the random seed. We learned the error in \eqref{eq:e} using GPR with the SE kernel \eqref{eq:K-SE}, the hyperparameters of which are obtained by marginal likelihood maximization \cite{rasmussen2003gaussian}. The training dataset has been generated by evaluating the true inverse dynamics model in \eqref{eq:Dynamics} on a trajectory of the same type of \eqref{eq:ref_trj}, with a seed different from those considered in the tracking experiment. To reduce the computational burden, the training data have been downsampled with a constant rate of $50$, which resulted in a dataset of $100$ samples. 

Regarding the implementation of the robust control law in \eqref{eq:Fed_Lin_Scheme_robust_prob}, the value of $\rho$ in \eqref{eq:rho} is obtained with $\beta = 3$, which is a common choice in literature \cite{helwa2019provably}. To stress modeling uncertainty, the nominal model has been chosen as $\hat{M} = 0.5\,I$, $\hat{n} =0$.
Finally, we used $\epsilon = 0.5$, which resulted in almost no chattering.

We compare our method against three baseline controllers. The first controller implements the feedback linearization law in \eqref{eq:Fed_Lin_Scheme} using the true model, which is included as a reference. The second controller exploits only the nominal model as in \eqref{eq:Fed_Lin_Scheme_uncertaint} with no robust term, while the third baseline consists of the nominal model plus error compensation, namely \eqref{eq:Fed_Lin_Scheme_robust_prob} with $w = 0$. 
For all the controllers, the proportional and derivative gains are diagonal matrices defined as  $K_P = k_p\, I$ and
$K_D = k_d\, I$, with $k_p = 50$ and $k_d = 2\sqrt{k_p}$, respectively. 


In Table~\ref{tab:MC_results} we report the trajectory tracking results obtained with the considered controllers. The performances are evaluated in terms of root mean squared error (RMSE) averaged over the 2 joints. The table reports the mean and the standard deviation of the average RMSE obtained over the 10 seeds. It can be seen that our approach (robust gp) strongly outperforms the nominal and gp baselines, which confirms the capability of the method to compensate for model inaccuracies and uncertainties.   
\begin{table}[h!]
    \centering
    \caption{RMSE distribution obtained on 10 trajectories, reported as mean $\pm$ standard deviation in degrees.}
    \begin{tabular}{ c c c c }
        \hline \\[-1.5ex]
        true & nominal & gp & robust gp\\ [0.5ex]
        \hline \\ [-1.5ex]
        $2.28 \pm 0.2$ & $116.80 \pm 6.43$ & $55.56 \pm 23.15$ & $15.67 \pm 3.89$\\
        \hline \\
    \end{tabular}
    \label{tab:MC_results}
\end{table}

To better understand the benefits of the proposed method, Fig.~\ref{fig:tracking_performance}
reports the joint positions, tracking errors, and actuation torques obtained on one of the reference trajectories. 
First, note that the nominal model, which is very inaccurate, is not able to track the reference trajectory, with a tracking error on the first joint that reaches the $100 \%$ of the desired trajectory. The addition of the error compensation strongly improves the performance. However, the non-robust GP model fails to track the reference when the state of the system reaches values far from those contained in the training dataset. As can be seen from the last few seconds of the reported trajectory, when this happens the GP outcome coincides with the prior mean, which is zero, and the controller action is the same as the nominal model's one. 
The addition of the robust term proposed in this work guarantees better tracking, with a control action that strongly resembles the true model's one over the entire simulation horizon. 
Finally, we note that the residual error obtained with the true model is due to the approximations induced by the discrete-time implementation.

\section{Conclusions}\label{sec:conclusions}

In this paper, we presented a novel GP-based robust control scheme that assures highly accurate trajectory tracking for the class of Lagrangian systems. The idea is to learn the modeling error with GPR and exploit the uncertainty provided to design an additional robust term, which, under mild assumptions, guarantees asymptotic error rejection. 
We validated the proposed approach on a simulated 2 DOF system. Results show that the developed method can compensate for model uncertainties and inaccuracies, with a control action comparable to the uncertainty-free model.
In future works, we plan to extend the evaluation to more complex and real manipulators.





\bibliographystyle{IEEEtran}
\bibliography{ref}

\end{document}